\documentclass{article}

\usepackage{microtype}
\usepackage{graphicx}
\usepackage{subfigure}
\usepackage{booktabs} 
\usepackage{algorithm}
\usepackage{algorithmic}
\renewcommand{\algorithmiccomment}[1]{\hfill $\triangleright$ #1}
\usepackage{amsthm}
\usepackage{microtype}
\usepackage{graphicx}
\usepackage{colortbl}
\usepackage{comment}

\usepackage{hyperref}

\usepackage[accepted]{icml2024}

\usepackage{amsmath}
\usepackage{amssymb}
\usepackage{mathtools}
\usepackage{amsthm}

\usepackage[capitalize,noabbrev]{cleveref}

\theoremstyle{plain}
\newtheorem{theorem}{Theorem}[section]

\newtheorem{lemma}[theorem]{Lemma}

\theoremstyle{definition}

\theoremstyle{remark}
\newtheorem{remark}[theorem]{Remark}

\usepackage[textsize=tiny]{todonotes}

\icmltitlerunning{An Effective Dynamic Gradient Calibration Method for Continual Learning}

\begin{document}

\twocolumn[
\icmltitle{An Effective Dynamic Gradient Calibration Method for Continual Learning}



\icmlsetsymbol{equal}{*}

\begin{icmlauthorlist}
\icmlauthor{Weichen Lin}{equal,yyy}
\icmlauthor{Jiaxiang Chen}{equal,yyy}
\icmlauthor{Ruomin Huang}{comp}
\icmlauthor{Hu Ding}{sch}

\end{icmlauthorlist}

\icmlaffiliation{yyy}{School of Data Science, University of Science
 and Technology of China, Anhui, China.}
\icmlaffiliation{comp}{Duke University} 
\icmlaffiliation{sch}{School of Computer Science and Technology, University of Science and Technology of China, Anhui, China}

\icmlcorrespondingauthor{Hu Ding}{huding@ustc.edu.cn}

\icmlkeywords{Machine Learning, ICML}

\vskip 0.3in
]



\printAffiliationsAndNotice{\icmlEqualContribution} 

\begin{abstract}
Continual learning (CL) is a fundamental topic in machine learning, where the goal is to train a model with continuously incoming data and tasks. Due to the memory limit, we cannot store all the historical data, and therefore confront the ``catastrophic forgetting'' problem, i.e., the performance on the previous tasks can substantially decrease because of the missing information in the latter period. Though a number of elegant methods have been proposed, the catastrophic forgetting phenomenon still cannot be well avoided in practice. In this paper, we study the problem from the gradient perspective, where our aim is to develop an effective algorithm to calibrate the gradient in each updating step of the model; namely, our goal is to guide the model to be updated in the right direction under the situation that a large amount of historical data are unavailable. Our idea is partly inspired by the seminal stochastic variance reduction methods (e.g., SVRG and SAGA)  for reducing the variance of gradient estimation in stochastic gradient descent algorithms. Another benefit is that our approach can be used as a general tool, which is able to be incorporated with several existing popular CL methods to achieve better performance. We also conduct a set of experiments on several benchmark datasets to evaluate the performance in practice. 
\end{abstract}

\section{Introduction}
\label{sec-introduction}
In the past years, Deep Neural Networks (DNNs) demonstrate remarkable performance for many different tasks in artificial intelligence, such as image generation \cite{ho2020denoising, goodfellow2014generative}, classification ~\cite{liu2021swin, he2016deep}, and pattern recognition~\cite{bai2021explainable,DBLP:journals/fcsc/ZhuYB16}. Usually we assume that the whole training data is stored in our facility and the DNN models can be trained offline by using Stochastic Gradient Descent (SGD) algorithms~\cite{bottou1991stochastic, bottou2018optimization}.  
However, real-world applications often require us to consider training lifelong models, where the tasks and data are accumulated in a streaming fashion \cite{van2019three, parisi2019continual}. For example, with the popularity of smart devices, a large amount of new data is generated every day. A  model needs to make full use of these new data to improve its performance while keeping old knowledge from being forgotten.  Those applications motivate us to  study the problem of {\em continual learning (CL)}~\cite{kirkpatrick2017overcoming, li2017learning}, where its goal is to develop effective method for gleaning insights from current data while retaining information from prior training data.


  A significant challenge that CL encounters is  ``{\em catastrophic forgetting}'' \cite{kirkpatrick2017overcoming,mccloskey1989catastrophic, DBLP:journals/corr/GoodfellowMDCB13}, wherein the exclusive focus on the current set of examples could result in a dramatic deterioration in the performance on previously learned data. This phenomenon is primarily attributed to limited storage and computational resources during the training process; otherwise, one could directly train the model from scratch using all the saved data. To address this issue, we need to develop efficient algorithm for training neural networks from a continuous stream of non-i.i.d. samples, with the goal of mitigating catastrophic forgetting while effectively managing computational costs and memory footprint.

  A number of elegant CL methods have been proposed to alleviate the catastrophic forgetting issue~\cite{wang2023comprehensive,de2021continual,mai2022online}. One representative CL approach is referred to as ``Experience Replay (ER)'' \cite{ratcliff1990connectionist, chaudhry2019tiny}, which has shown promising performance in several 
continual learning scenarios \cite{DBLP:conf/cvpr/PrabhuHDTLGB23, DBLP:conf/iclr/AraniSZ22, farquhar2018towards}. Roughly speaking, the ER method utilizes reservoir sampling~\cite{vitter1985random} to maintain historical data in the buffer, then extract new incoming training data with random samplings for learning the current task. 
Though the intuition is simple, the ER method currently is one of the  most popular CL approaches that incurs 
moderate computational and storage demands.
Moreover, several recently proposed approaches suggest that the ER method can be combined with knowledge distillation to further improve the performance; for example, 
the methods of DER/DER++ \cite{buzzega2020dark} and X-DER \cite{boschini2022class} preserve previous training samples alongside their logits in the model as the additional prior knowledge.
Besides the ER methods, there are also several other types of CL techniques proposed in recent years, and please 
refer to \cref{sec-related} for a detailed introduction.

\subsection{Our Main Ideas and Contributions}
\label{sec-ours}
  Though existing CL methods can alleviate the catastrophic forgetting issue from various aspects, the practical performances in some scenarios are still not quite satisfying \cite{DBLP:journals/tmlr/YuHHLWL23, tiwari2022gcr, DBLP:conf/cvpr/GhunaimBAAHPTG23}.
  In this paper, 
we study the continual learning problem from the gradient perspective, and the rationality behind is as follows. In essence, an approach for avoiding the catastrophic forgetting issue in CL,  e.g., the replay mechanisms or the regularization strategies, ultimately manifests its influence on the gradient directions during model updating~\cite{wang2023comprehensive}. If all the historical data are available, one could compute the gradient by using the stochastic gradient descent method and obviously the catastrophic forgetting phenomenon cannot happen. The previous  methodologies  aim to approximate the gradient by preserving additional information and incorporating it as a constraint to model updates, thereby retaining historical knowledge. However,   the replay-based methods in practice are often limited by storage capacity, which leads to a substantial loss of historical data information and inaccurate estimation of historical gradients \cite{yan2021dynamically}.
Therefore, our goal is to develop a more accurate gradient calibration algorithm in each step of the continual learning procedure, which can directly enhance the training quality. 

 We are aware of several existing CL methods that also take into account of the gradients~\cite{DBLP:conf/iclr/LiuL22, tiwari2022gcr, farajtabar2020orthogonal}, but our idea proposed here is fundamentally different. 
 We draw inspiration from the seminal {\em stochastic variance reduction} methods (e.g., SVRG from \citet{johnson2013accelerating} and SAGA from \citet{defazio2014saga}), which are originally designed to reduce the gradient variance so that the estimated gradient can closely align with the true full gradient over the entire dataset (including the current and historical data). These variance reduction methods have been extensively studied in the line of the research on stochastic gradient descent method~\cite{jin2019towards,babanezhad2015stopwasting,lei2017non}; their key idea is to leverage the additional saved full gradient information to calibrate the gradient in the current training step, which leads to significantly reduced gradient variance comparing with the standard SGD method.  This intuition also inspires us to handle the CL problem. In a standard SGD method, the variance between the obtained gradient and the full gradient is due to the ``batch size limit" (if the batch size has no bound, we can simply compute the full gradient).  Recall that the challenge of CL is due to the ``buffer size limit'', which impedes the use of full historical data (this is similar with the dilemma encountered by SGD with the ``batch size limit''). So an interesting question is 
 
\hspace{0.2in} \textit{Can the calibration idea for ``batch size limit'' be modified to handle ``buffer size limit''? Specifically,  is it possible to develop an effecitve method to compute a SVRG (or SAGA)-like calibration for the gradient in  CL scenarios?}

  Obviously, it is challenging to directly implement the SVRG or SAGA algorithms in continual learning because of the missing historical data. 
Note that \citet{frostig2015competing} proposed a streaming SVRG (SSVRG) method that realizes the SVRG method within a given fixed buffer, but unfortunately it does not perform quite well in the CL scenarios (as shown in our experimental section). One possible reason is that SSVRG can only leverage information within the buffer and fails to utilize all historical information. 

In this paper, we aim to apply the intuition of SVRG to handle the ``buffer size limit'' in CL, and our contributions can be summarized as follows:

\begin{itemize}
    \item First, we propose a novel two-level dynamic algorithm, named \textbf{Dynamic Gradient Calibration (DGC)},   to maintain a gradient calibration in continual learning. DGC can effectively tackle the storage limit and leverage historical data to calibrate the gradient of the model at each current stage. 
   Moreover, our theoretical analysis shows that our DGC based CL algorithm can achieve a linear convergence rate. 

    \begin{figure}[t]
    \centering
    \includegraphics[width=0.7\linewidth]{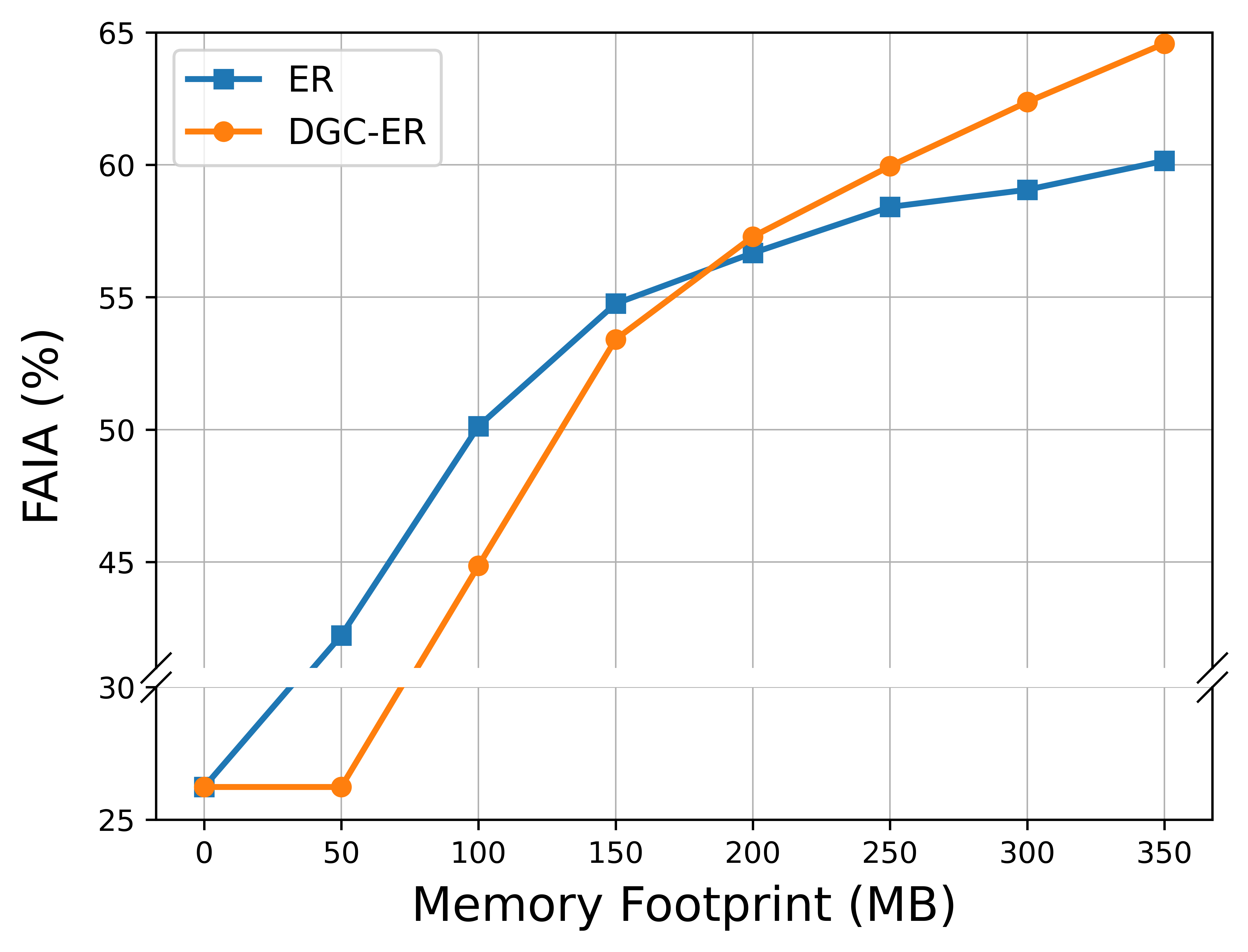} 
    \caption{Performance comparison of the ER method and our proposed DGC-ER method with increasing storage memory size on CIFAR100.  The x-axis denotes the actual memory footprint (MB) and the y-axis denotes the final Average Incremental Accuracy (FAIA) ~\cite{hou2019learning, douillard2020podnet} (the formal definition is shown in \cref{Sec: Experiment}). The curves in the figure indicate that   although ER outperforms DGC-ER at low MB, as MB increases, the dynamic gradient calibration method can achieve a larger marginal benefit than simply increasing the sample size of ER.  } 
    \label{fig: MemorySize vs AIA}  
    \vskip -0.2in
    \end{figure}
   
    \item Second, our method can be conveniently integrated with most existing reservoir sampling-based continual learning approaches (e.g., ER \cite{ratcliff1990connectionist}, DER/DER++ \cite{buzzega2020dark}, XDER \cite{boschini2022class}, and Dynamic ER \cite{yan2021dynamically}), where this hybrid framework can 
    induce a significant enhancement to the overall model performance. 
    Note that storing the gradient calibrator can cause an extra memory footprint; so a key question is whether such an extra memory footprint can yield a larger marginal benefit than simply taking more samples to fill the extra memory? 
    In Fig \ref{fig: MemorySize vs AIA}, we illustrate a brief example to answer this question in the affirmative; more detailed evaluations are shown in the experimental part. 
    
    \item Finally, we conduct a set of comprehensive experiments on the popular datasets S-CIFAR10, S-CIFAR100 and S-TinyImageNet; the experimental results suggest that our method can improve the final Average Incremental Accuracy (FAIA) in several CL scenarios by more than $6\%$. Moreover, our improvement for a larger number of tasks is more significant than that for a smaller number. Furthermore, our adoption of the SVRG-inspired DGC calibration method leads to enhanced stability in minimizing the loss function throughout the parameter optimization process.
\end{itemize}
\subsection{Related Work}
\label{sec-related}
We briefly overview existing important continual learning approaches (except for the ones mentioned before).  We also refer the reader to the recent surveys~\cite{wang2023comprehensive,de2021continual,mai2022online} for more details.


  A large number of CL methods are replay based, where they often keep a part of previous data through approaches like reservoir sampling \cite{DBLP:conf/iclr/ChaudhryRRE19, DBLP:conf/iclr/RiemerCALRTT19}. Several more advanced data selection strategies focus on optimizing the factors like the sample diversity of parameter gradients or the similarity to previous gradients on passed data, e.g.,  GSS \cite{aljundi2019gradient} and   GCR \cite{tiwari2022gcr}.  
  Experience replay can be effectively combined with knowledge distillation. For example, \citet{hu2021distilling} proposed to distill colliding effects from the features for new coming tasks, and ICARL~\cite{rebuffi2017icarl} proposed to take account of the data representation trained on old data. MOCA \cite{DBLP:journals/tmlr/YuHHLWL23} improves replay-based methods by diversifying representations  in the  space. Another replay-based approach is based on generative replay, which obtains replay data by generative models~\cite{shin2017continual, DBLP:conf/icml/GaoL23a, wu2018memory}.


  Another way for solving continual learning is through some deliberately designed optimization procedures. For example, the methods GEM \cite{lopez2017gradient}, AGEM \cite{DBLP:conf/iclr/ChaudhryRRE19}, and MER \cite{DBLP:conf/iclr/RiemerCALRTT19} restrict parameter updates to align with the experience replay direction, and thereby preserve the previous input and gradient space with old training samples. Different from saving old training samples, \citet{farajtabar2020orthogonal} proposed to adapt parameter updates in the orthogonal direction of the previously saved gradient. {  The method AOP \cite{guo2022adaptive} projects the gradient in the direction orthogonal to the subspace spanned by all previous task inputs, therefore it only keeps an orthogonal projector rather than storing previous data.}



  To mitigate the problem of forgetting, we can also augment the model capacity for learning new tasks. \citet{xu2018reinforced} tried to enhance model performance by employing meta-learning techniques  when dynamically extending the model.
The method ANCL \cite{DBLP:conf/cvpr/KimNOH23} proposes to utilize an auxiliary network to achieve a trade-off between 
plasticity 
and 
stability. Dynamic ER \cite{yan2021dynamically} introduces a novel two-stage learning method that employs a dynamically expandable representation for learning knowledge incrementally. 

\section{Preliminaries}
We consider the task of training a soft-classification function $\boldsymbol{f(\cdot ;\theta)$: $\mathcal{X}\to \mathcal{Y}}$, where $\mathcal{X}$ and $\mathcal{Y}$ respectively represent the space of data and the set of labels, and $\theta$ is the parameter to optimize. Without loss of generality, we assume $\mathcal{Y}=\{1, 2, \cdots, K\}$. So 
 $f(\cdot; \theta)$ maps each $x\in\mathcal{X}$ to some $f(x; \theta)\in\mathbb{R}^K$. 
 To find an appropriate $\theta$, the classification function  $f$ is usually equipped with a loss function $\ell (f(x;\theta), y)$, which is differentiable for the variables $x$ and $\theta$ (e.g., cross-entropy loss). To simplify the notation, we use $ \ell(x,y,\theta)$ to denote $ \ell(f(x;\theta), y)$.  Given a set of data $P=\{(x^i, y^i)\mid 1\leq i\leq n\} = \mathcal{X}_P\times \mathcal{Y}_P \subset \mathcal{X}\times \mathcal{Y}$, the training process is to find a $\theta$ such that the empirical risk of $\ell(x,y, \theta)$, i.e., $\sum^n_{i=1}\ell(x^i,y^i, \theta)$, is minimized. We define the full gradient of $ \ell(x,y, \theta)$ on $P$ as 
\begin{align}\label{def: full gradient}
    \mathcal{G}(P, \theta) & \triangleq \nabla_\theta \ell(\mathcal{X}_P, \mathcal{Y}_P, \theta) = \frac{1}{n} \sum^n_{i=1} \nabla_\theta \ell(x^i, y^i, \theta). 
\end{align}

\subsection{Continual Learning Models}
\label{subsec:CiCL}
In this paper, we focus on two popular CL models: {\em Class-Incremental Learning (\textbf{CIL})}~\cite{hsu2018re} and {\em Task-Free Continual Learning (\textbf{TFCL})}~\cite{aljundi2019task}.

  In the setting of CIL, the training tasks come in a sequence $\{\mathcal{T}_1, \mathcal{T}_2, \cdots, \mathcal{T}_T\}$ with disjoint label space; each time spot $t \in \{1,2,\cdots,T\}$ corresponds to the task $\mathcal{T}_t$ with a training dataset $\{(x_t^i,y_t^i)\mid 1\leq i\leq n_t\}$. With a slight abuse of notations, we also use $\mathcal{T}_t$ to denote its training dataset. Also, we use $\mathcal{Y}_t$ to denote the corresponding set of labels $\{y_t^i\mid 1\leq i\leq n_t\}$. 
Although the training data for individual tasks $\mathcal{T}_t$ is independently and identically distributed (i.i.d.), it is worth noting that the overall task stream $\{\mathcal{T}_1, \mathcal{T}_2, \cdots, \mathcal{T}_T\}$ does not adhere to the i.i.d. assumption due to the evolving label space over time.  We define the overall risk under this CL setting 
  at the current time spot $t$ 
 as follows: 
\begin{align}
\ell_{CL}^t (\theta) \triangleq \frac{1}{t} \sum_{c = 1}^{t} \underset{(x, y) \sim \mathcal{T}_{c}}{\mathbb{E}}[\ell(x,y,\theta)]. 
\end{align}

  The setting of TFCL is similar to CIL, where the major difference is that the task identities are not provided in neither the training nor testing procedures. So the TFCL setting is more challenging than CIL because the algorithm is unaware of task changes and the current task identity. In the main part of our paper, we present our results in CIL; in our supplement, we  explain how to extend our results to TFCL.


\subsection{Variance Reduction Methods}
\label{subsec:SVRG}
A comprehensive introduction to variance reduction methods is provided in~\cite{gower2020variance}. Here 
we particularly introduce 
one representative variance reduction method \textbf{SVRG}~\cite{johnson2013accelerating} 
, which is closely related to our proposed CL approach. 

The high-level idea of SVRG is to construct a calibration term to reduce the variance in the gradient estimate. The complete optimization process can be segmented into a sequence of stages. We denote the training data as $P = \{(x^i,y^i)\}_{i=1}^{n} \subset \mathcal{X}\times \mathcal{Y}$,  and denote the parameter at the beginning of each stage as $\tilde{\theta}$.  The key part of SVRG is to  minimize the variance in SGD optimization by computing 
an additional term $\tilde{\mu}$:
\begin{align}
\label{for-svrgu}
\tilde{\mu}  &  \triangleq  \mathcal{G}(P, \tilde{\theta}) =  \frac{1}{n} \sum_{i = 1}^{n} \nabla_\theta \ell(x^i,y^i, \tilde{\theta}). 
\end{align}
In each stage, 
SVRG applies the standard SGD with the term $\tilde{\mu}$ to update the parameter $\theta$: 
let
\begin{align}\label{eq:SVRG}
v^k  &=  \nabla_\theta \ell(x^i, y^i,\theta^k) - (\nabla_\theta \ell(x^{i}, y^{i}, \tilde{\theta}) -\tilde{\mu}), \\
\theta^{k+1} & = \theta^{k} - \eta v^k,
\end{align}
where $(x^i,y^i) \in P$ is the sampled training data, $\theta^k$ represents the  parameter at the $k$-th step of SGD, and $\eta$ is the learning rate. Since $\mathbb{E}[\nabla_\theta \ell(x^{i}, y^{i}, \tilde{\theta^k})] = \tilde{\mu}$, $v^k $ is an unbiased estimate of the gradient $\mathcal{G}(P,\theta^k)$. Subsequently, the term ``$ \nabla_\theta \ell(x^{i}, y^{i}, \tilde{\theta}) - \tilde{\mu}$'' in Eq (\ref{eq:SVRG}) can be regarded as a \textbf{calibrator} to reduce the variance of gradient estimation and achieve a linear convergence rate \cite{johnson2013accelerating}, which is faster than directly using $\nabla_\theta \ell(x^i, y^i,\theta^k)$.


\section{Our Proposed Method}
\label{Sec: DGC}
In this section, we propose 
the Dynamic Gradient Calibration (DGC) approach which maintains a gradient calibration during the learning process. A highlight of DGC is that it utilizes the whole historical information to obtain a more precise gradient estimation, and consequently relieves the negative impact of catastrophic forgetting. In \cref{subsec: naive approach}, we introduce ER and analyze the obstacle if we directly combine ER with SVRG. 
In \cref{subsec: SVRG estimator}, we present our DGC  method for addressing the issues discussed in \cref{subsec: naive approach}. In \cref{subsec: Combine DGC}, we explain how to integrate 
  DGC with other  CL  techniques.

\subsection{Experience Replay Revisited and SVRG}
\label{subsec: naive approach}
 First, we overview the classical ER \cite{ratcliff1990connectionist, chaudhry2019tiny} method as a baseline for the CIL setting. ER employs the reservoir sampling algorithm to dynamically manage a \textbf{buffer} (denoted as $\boldsymbol{\mathcal{M}_t}$) at time $t$, which serves to store historical data. At each time spot $t$ (and assume the current updating step number of the optimization is $k$), ER updates the model parameter  $\theta^{k}_t$ following the standard gradient descent method: 
\begin{align}
    \theta_t^{k+1} = \theta_t^{k} - \eta \cdot v_t^k, 
    \label{eq:update}
\end{align}
where $\eta$ is the learning rate and $v_t^k$ is the calculated gradient. If not using any replay strategy, $v_t^k$ is usually calculated on a randomly sampled training data $(x^k_t, y^k_t) \in  \mathcal{T}_t$. It is easy to see that this simple strategy can cause the forgetting issue for shifting data stream since it does not contain any information from the previous data. 
Hence the classical ER algorithm takes a random sample 
$(\bar{x}^k_t, \bar{y}^k_t)$ from the aforementioned buffer $\mathcal{M}_t$ (who contains a subset of historical data via reservoir sampling), and computes the gradient:
\begin{eqnarray}\label{eq:ER-graidient}
    v_t^k =\frac{1}{t}  \nabla_\theta \ell(x^k_t, y^k_t, \theta_t^k) + \frac{t-1}{t} \nabla_\theta \ell(\bar{x}^k_t, \bar{y}^k_t, \theta_t^k).
\end{eqnarray}
\begin{remark}
\label{rem-er}
\textbf{(1)} For simplicity, we assume that the data sets of all the tasks have the same size. So the obtained $v_t^k$ in (\ref{eq:ER-graidient}) is an unbiased estimation of the full gradient $\mathcal{G}(\bigcup_{c=1}^{t}\mathcal{T}_c,\theta)$ at the current time spot $t$. If they have different sizes, we can simply replace the coefficients ``$\frac{1}{t}$'' and ``$\frac{t-1}{t}$'' by ``$\frac{|\mathcal{T}_t|}{\sum^t_{c=1}|\mathcal{T}_c|}$'' and ``$\frac{\sum^{t-1}_{c=1}|\mathcal{T}_c|}{\sum^t_{c=1}|\mathcal{T}_c|}$'', respectively. \textbf{(2)} Also, we assume that the batch sizes of the random samples from $\mathcal{T}_t$ and $\mathcal{M}_t$ are both ``$1$'', i.e., we only take single item $(x^k_t, y^k_t)$ and $(\bar{x}^k_t, \bar{y}^k_t)$ from each of them. Actually, we can also take larger batch sizes and then the Eq (\ref{eq:ER-graidient}) can be modified correspondingly by taking their average gradients. 
\end{remark}



  Now we attempt to apply the SVRG method to Eq (\ref{eq:ER-graidient}).  
Our objective is to identify a more accurate unbiased estimate of the gradient at the current time spot $t$ so as to determine the updating direction.
At first glance, one possible solution is to adapt the streaming SVRG method \cite{frostig2015competing} to the CL scenario. 
We treat $\mathcal{M}_t$ as the static data set in our memory, and apply the 
 SVRG technique to calibrate the gradient ``$\nabla_\theta \ell(x^k_t, y^k_t, \theta_t^k) $'' and ``$ \nabla_\theta \ell(\bar{x}^k_t, \bar{y}^k_t, \theta_t^k)$'' in Eq (\ref{eq:ER-graidient}). 
 Similar to the procedure introduced in \cref{subsec:SVRG}, we denote the parameter at the beginning of the current stage $s$ as $\tilde{\theta}_{t,s}$. For simplicity, when we consider the update within stage $s$, we just use $\tilde{\theta}_{t}$ to denote $\tilde{\theta}_{t,s}$. Similar with Eq (\ref{for-svrgu}), we define the  terms
\begin{align}
\tilde{\mu} \triangleq  \mathcal{G}(\mathcal{M}_t, \tilde{\theta}_t), \quad \tilde{v} \triangleq \mathcal{G}(\mathcal{T}_t, \tilde{\theta}_t).
\end{align}
Then we can calibrate the gradient by $\nabla_\theta \ell({x}_t,{y}_t,\tilde{\theta}_t)$ and $\nabla_\theta \ell(\bar{x}_t, \bar{y}_t, \tilde{\theta_t})$ (which serves as the similar role of $\nabla_\theta \ell(x^{i}, y^{i}, \tilde{\theta}$ in (\ref{eq:SVRG})); the new form of $v_t^k $ becomes
\begin{align} \label{Eq: SSVRG}
v_t^k & =\frac{1}{t}   \left(\nabla_\theta \ell(x_t, y_t, \theta_t^k) - \nabla_\theta \ell({x}_t,{y}_t,\tilde{\theta}_t) + \tilde{v}  \right) + \nonumber\\ 
& \frac{t-1}{t} \left( \nabla_\theta \ell(\bar{x}_t, \bar{y}_t, \theta_t^k) -  \nabla_\theta \ell(\bar{x}_t, \bar{y}_t, \tilde{\theta_t}) + \tilde{\mu}\right). 
\end{align}

Obviously, if the buffer $\mathcal{M}_t$ contains the whole historical data (denote by $\mathcal{T}_{[1:t)} = \bigcup_{c=1}^{t-1}\mathcal{T}_c$), the above approach is exactly the standard SVRG. However, because $\mathcal{M}_t$ only takes a small subset of $\mathcal{T}_{[1:t)}$, this approach still cannot avoid information loss for the previous tasks. In next section, we propose a novel two-level dynamic algorithm to record more useful information from $\mathcal{T}_{[1:t)}$, and thereby reduce the information loss induced by $\mathcal{M}_t$. We also take the approach of Eq (\ref{Eq: SSVRG}) as a baseline in \cref{Sec: Experiment} to illustrate the advantage of our proposed approach. 

\subsection{Dynamic Gradient Calibration}
\label{subsec: SVRG estimator}
To tackle the issue discussed in \cref{subsec: naive approach},  we propose a novel two-level update approach ``Dynamic Gradient Calibration (DGC)'' to maintain our calibration term. Our focus is designing a method to incrementally update an unbiased estimation for $\mathcal{G}(\mathcal{T}_{[1:t)}, \theta_t^k)$. To illustrate our idea clearly, we decompose our analysis to two levels: \textbf{(1)} update the parameter during the training within each time spot $t$; \textbf{(2)} update the parameter at the transition from time spot $t$ to  $t+1$ (i.e., the moment that the task $\mathcal{T}_t$ has just been completed and the task $\mathcal{T}_{t+1}$ is just coming).



  \textbf{(1) How to update the parameter during the training within each time spot $t$.} We follow the setting of the streaming SVRG as discussed in \cref{subsec: naive approach}: 
  the training process at the current time spot $t$ is divided into a sequence of stages;  the model parameter at the beginning of each stage is recorded as $\tilde{\theta}_t$. 
  To illustrate our idea for calibrating the gradient $v^k_t$ in Eq (\ref{Eq: SSVRG}), we begin by considering an ``imaginary'' approach: 
  we let
\begin{align} \label{eq: final update}
    v_t^k & =\frac{1}{t}  \left( \nabla_\theta \ell(x_t, y_t, \theta_t^k) -\nabla_\theta \ell( {x}_t, {y}_t,\tilde{\theta}_t) + \mathcal{G}(\mathcal{T}_t,\tilde{\theta}_t)\right) \nonumber\\
    & + \frac{t-1}{t} \Gamma, 
\end{align}
 where 
\begin{align} \nonumber
     \Gamma= \nabla_\theta \ell(\bar{x}_t, \bar{y}_t, \theta_t^k)- \underbrace{  \left(\nabla_\theta \ell(\bar{x}_t, \bar{y}_t, \tilde{\theta}_t)  - \mathcal{G}(\mathcal{T}_{[1:t)}, \tilde{\theta}_t )\right). }_{\text{Calibration from the previous parameter $\tilde{\theta}_t$}} 
\end{align}
Different from Eq (\ref{Eq: SSVRG}), we compute $v^k_t$ based on the full historical data $\mathcal{T}_{[1:t)}$, which follows the same manner of SVRG. However, a major obstacle here is that we cannot obtain the exact $\Gamma$ since $\mathcal{T}_{[1:t)}$ is not available. This motivates us to design a relaxed form of (\ref{eq: final update}).
We define a surrogate function to approximate $\Gamma$, which can be computed through recursion. Suppose each training stage has $m\geq 1$ steps, then we define
\begin{align} 
 \nonumber \Gamma_{\text{DGC}}(t, k) & = \nabla_\theta \ell(\bar{x}_t, \bar{y}_t, \theta_t^k) - \\ 
 \label{eq: DGC inner update}  &\left(\nabla_\theta \ell(\bar{x}_t, \bar{y}_t, \tilde{\theta}_t)  - \Gamma'_{\text{DGC}}(t) \right), \\ 
 \label{eq: stage update} \Gamma'_{\text{DGC}}(t) & = \Gamma_{\text{DGC}}(t, m+1),
\end{align}
 in the stage. Note that the term ``$\nabla_\theta \ell(\bar{x}_t, \bar{y}_t, \tilde{\theta}_t)$'' in (\ref{eq: DGC inner update}) can be computed by the previous parameter $\tilde{\theta}_t$ during training, so we do not need to store it in buffer. For the initial  $t=1$ case (i.e., when we just encounter the first task),  we can directly set $\Gamma'_{\text{DGC}}(1)=\Vec{0}$. We update the function $\Gamma'_{\text{DGC}}(t)$ at the end of each training stage in (\ref{eq: stage update}), and use the function $\Gamma_{\text{DGC}}(t, k)$ to approximate $\Gamma$ in (\ref{eq: final update}). Comparing with the original formulation of $\Gamma$ in (\ref{eq: final update}), we only replace the term ``$\mathcal{G} (\mathcal{T}_{[1:t)}, \tilde{\theta}_t )$'' by ``$\Gamma'_{\text{DGC}}(t)$''. Also, we have the following lemma to support this replacement.  The detailed proof of lemma \ref{lem-gamma} is provided in our supplement.

\begin{lemma}
\label{lem-gamma}
\begin{align}
 \mathbb{E} \left[\Gamma'_{\text{DGC}}(t) \right] = \mathcal{G} (\mathcal{T}_{[1:t)}, \tilde{\theta}_t )
\end{align}
\end{lemma}

   We utilize the term ``$\nabla_\theta \ell(\bar{x}_t, \bar{y}_t, \tilde{\theta}_t)  - \Gamma'_{\text{DGC}}(t)$'' of (\ref{eq: DGC inner update}) as the calibrator for each updating step, thereby preserving the unbiased nature of the gradient estimator and reducing the variance of gradient estimation.

  \textbf{(2) How to update the parameter at the transition from  time spot $t$ to $t+1$.} 
 At the end of  time spot $t$, we update the recorded $\tilde{\theta}_t$ to $\theta_t^{m+1}$, and the data $\mathcal{T}_t$ from time $t$ should be integrated into the historical data. In this context, it is essential to update the calibrated gradient accordingly:
\begin{equation}
\begin{aligned} \label{eq: DGC outer update}
\Gamma'_{\text{DGC}}(t+1) = \frac{1}{t}\left((t-1) \cdot \Gamma'_{\text{DGC}}(t) +  \mathcal{G}(\mathcal{T}_t, \tilde{\theta}_t) \right) .
\end{aligned}
\end{equation}

 The complete algorithm is presented in Algorithm~\ref{alg}. Compared with the conventional reservoir sampling based approaches, we only require the additional   
  storage for keeping $\Gamma'_{\text{DGC}}(t)$, and so that the gradient $\Gamma_{\text{DGC}}(t, k)$ in (\ref{eq: DGC inner update}) can be effectively updated by the recursion. Moreover, our method can also conveniently adapt to TFCL where the task boundaries are not predetermined. In such a setting, we can simply treat each batch data (during the SGD) as a ``micro'' task at the time point and then update the gradient estimation via Eq (\ref{eq: DGC outer update}).
  The detailed algorithm for TFCL is placed in our appendix. 

Similar to the theoretical analysis of SVRG~\cite{johnson2013accelerating}, under mild assumptions, the optimization procedure of our DGC method at each time spot $t$ can also achieve linear convergence. We denote the optimal parameter at time spot $t$ as $\theta_{*} \triangleq \arg\min_{\theta} \ell_{\text{CL}}^{t}(\theta)$. Then we have the following theorem. 

\begin{theorem} \label{th:main} 
 Assume that $f(x;\theta)$ is  $L$-smooth and $\gamma$-strongly convex; the parameters 
$m \geq \frac{10L^2}{\gamma^2}$ and $\eta = \frac{\gamma}{10L}$. Then we have a linear convergence in expectation for the DGC procedure at time $t$:
$$\mathbb{E}\Big[\left\|\tilde{\theta}_{t,s+1}-\theta_{*}\right\|_{2}^{2}\Big] \leq \frac{1}{2^s}\mathbb{E}\Big[\left\|\tilde{\theta}_{t,1}-\theta_{*}\right\|_{2}^{2}\Big]$$
where $\tilde{\theta}_{t,s}$ represents the initialization parameter at the beginning of the $s$-th stage at time spot $t$.
\end{theorem}

  The  proof of Theorem~\ref{th:main} is provided in appendix. This theorem indicates that the gradient calibrated by our DGC method shares the similar advantages with SVRG. For instance, when updating each task $\mathcal{T}_t$, the loss function has a smoother decrease (we validate this property in \cref{subsec:loss trajectories}).

\begin{algorithm}[h t]
    \caption{DGC   procedure}
    \label{alg}
    \begin{algorithmic}[1]
        \STATE {\bfseries Input:} Data stream $\{\mathcal{T}_1, \mathcal{T}_2, \cdots\, \mathcal{T}_T\}$, update steps $m$, update stages $S$, batch size $b$, and learning rate $\eta$.
        \renewcommand{\algorithmiccomment}[1]{{\color{blue}\textit{/* #1 */}}}
        \STATE {\bfseries Output:} Trained model parameter $\tilde{\theta}_T$

        \STATE Initialize model parameters $\tilde{\theta}_0$, $\Gamma^\prime_{\text{DGC}}(1) = \Vec{0}$
        \STATE Initialize buffer $\mathcal{M}_1 = \emptyset$
        \FOR{$t = {1,2,\ldots, T}$}
            \STATE $\tilde{\theta}_t \gets \tilde{\theta}_{t-1}$
            \FOR{$s = {1,2,\ldots, S}$}
            \STATE $\theta_t^1 \gets \tilde{\theta}_t$
                \FOR{$k = {1,2,\ldots, m}$}
                \STATE Take a uniform sample $X_t$ of size $b$ from $\mathcal{T}_t$ 
                \STATE Take a uniform sample $X'_t$ of size $b$ from $\mathcal{M}_t$
                \STATE Calculate $\Gamma_{\text{DGC}}(t,k)$ with $\Gamma^\prime_{\text{DGC}}(t)$ according to   (\ref{eq: DGC inner update})
                
                \renewcommand{\algorithmiccomment}[1]{{\color{blue}\textit{/* #1 */}}}
                \STATE Calculate $v_t^k$ with $\Gamma_{\text{DGC}}(t,k)$  according to (\ref{eq: final update})
                
                \COMMENT{Calculate the calibrated gradient}
                \STATE $\theta_t^{k+1} \gets \theta_t^k - \eta \cdot v_t^k$
                \ENDFOR 
            \renewcommand{\algorithmiccomment}[1]{{\color{blue}\textit{/* #1 */}}}
            \STATE Update $\Gamma^\prime_{\text{DGC}}(t)$ according to  (\ref{eq: DGC inner update}) and (\ref{eq: stage update})
            
            \COMMENT{Update $\Gamma^\prime_{\text{DGC}}(t)$ from $\Gamma_{\text{DGC}}(t, m+1)$}
            \STATE $\tilde{\theta}_t \gets \theta_t^{m+1}$
            \ENDFOR
            \STATE $\mathcal{M}_{t+1} \gets$ MemoryUpdate($\mathcal{T}_t$, $\mathcal{M}_t$)
            \renewcommand{\algorithmiccomment}[1]{{\color{blue}\textit{/* #1 */}}}
            
            \COMMENT{Reservoir sampling}

            \STATE Calculate and store $\Gamma^\prime_{\text{DGC}}(t+1)$ according to (\ref{eq: DGC outer update})
            
            \COMMENT{Update $\Gamma^\prime_{\text{DGC}}(t+1)$ from $\Gamma^\prime_{\text{DGC}}(t)$ }
        \ENDFOR
    \end{algorithmic}
\end{algorithm}

\subsection{Combine DGC with Other CL Methods}
\label{subsec: Combine DGC}
Our proposed DGC approach can be also efficiently combined with other CL methods. 
As discussed in Section~\ref{sec-introduction}, a number of popular CL methods rely on the reservoir sampling technique to preserve historical data in   buffer~\cite{buzzega2020dark, boschini2022class, DBLP:conf/iclr/RiemerCALRTT19}. 
For these methods, such as DER and XDER, we can conveniently combine the conventional batch gradient descent with 
the DGC calibrated gradient estimator $\Gamma_{\text{DGC}}(t,k)$ defined in \cref{subsec: SVRG estimator}, so as to obtain a more precise gradient estimator with reduced variance based on Eq (\ref{eq: final update}):
\begin{eqnarray} \label{eq: combine}
    v_t^k &=\frac{1}{t}   \left(\nabla_\theta \ell(x_t, y_t, \theta_t^k) \!\!- \!\!\nabla_\theta(\bar{x}_t,\bar{y}_t,\tilde{\theta}_t) + \mathcal{G}(\mathcal{T}_t,\tilde{\theta}_t)  \right)\!\!+ \nonumber\\
    &  \frac{t-1}{t}\bigg[ \alpha \Gamma_{\text{DGC}}(t,k)\!\! +\!\! (1-\alpha)\nabla_\theta \ell(\bar{x}_t, \bar{y}_t, \theta^k_t)\bigg],
\end{eqnarray}
where $\alpha$ is a given parameter to control the proportion of the two unbiased estimations $\Gamma_{\text{DGC}}(t,k)$ and $\nabla_\theta \ell(\bar{x}_t, \bar{y}_t, \theta^k_t)$ of the gradient $\mathcal{G}(\mathcal{T}_{[1:t)},\theta_{t}^k)$. 
According to the theoretical analysis in SSVRG \cite{frostig2015competing}, 
the selection of $\alpha$ should be related to $1/L$, where $L$ is the smoothness coefficient of the model $f(x;\theta)$, i.e., 
\begin{equation}
    L = \max_{\theta_1,\theta_2 \in \Theta} \frac{|\nabla_{\theta} f(x;\theta_1) - \nabla_{\theta} f(x;\theta_2)|}{|\theta_1-\theta_2|},
\end{equation}
where $\Theta$ represents the parameter space. The experimental study on the impact of $\alpha$ is placed in our supplement. In \cref{Sec: Experiment}, we show that the amalgamation of the CL method and our DGC calibration procedure can yield a more precise update direction, and consequently enhance the ultimate model performance. 



\begin{table*}[htbp]
\centering
\begin{tabular}{ccccccc}
\toprule[2pt]
\textbf{Datasets} &
\multicolumn{2}{c}{\textbf{S-CIFAR10}} &
\multicolumn{2}{c}{\textbf{S-CIFAR100}} & 
\multicolumn{2}{c}{\textbf{S-TinyImageNet}}\\
Size of Buffer & $500$ & $2000$ & $500$ & $2000$ & $2000$ & $5000$ \\
\hline
\textsc{AOP} & \multicolumn{2}{c}{$66.73_{\pm0.60}$
} & \multicolumn{2}{c}{$42.73_{\pm0.62}$} & \multicolumn{2}{c}{$21.40_{\pm
0.17}$}\\
\textsc{AGEM} & $45.42_{\pm0.81}$ & $45.58_{\pm0.64}$ & $26.11_{\pm0.09}$ & $26.13_{\pm0.07}$ & $22.41_{\pm0.11}$ & $21.98_{\pm0.36}$\\
\textsc{SSVRG} & $49.02_{\pm4.09}$ & $58.68_{\pm3.77}$ & $27.74_{\pm2.45}$ & $39.09_{\pm5.53}$ & $14.39_{\pm1.47}$ & $15.88_{\pm2.37}$\\
\textsc{MOCA} & $81.01_{\pm0.97}$ & $85.06_{\pm0.51}$ & $54.14_{\pm0.43}$ & $59.29_{\pm2.97}$ & $34.74_{\pm10.08}$ & $38.86_{\pm8.58}$ \\
\textsc{GSS} & $68.81_{\pm0.98}$ & $76.08_{\pm1.35}$ & $33.72_{\pm0.22}$ & $38.54_{\pm0.39}$ & $31.38_{\pm0.11}$ & $34.31_{\pm0.22}$\\
\textsc{GCR} & $\underline{82.31}_{\pm0.43}$ & $86.35_{\pm0.48}$ & $53.43_{\pm2.15}$ & $63.18_{\pm2.26}$ & $48.94_{\pm0.44}$ & $54.60_{\pm0.43}$\\
\textsc{HAL} & $58.06_{\pm1.90}$ & $69.53_{\pm2.55}$ & $24.85_{\pm0.91}$ & $28.05_{\pm1.90}$ & $18.66_{\pm1.02}$ & $21.45_{\pm0.91}$\\
\textsc{ICARL} & $65.89_{\pm2.74}$ & $75.94_{\pm0.84}$ & $60.58_{\pm0.50}$ & $64.03_{\pm0.41}$ & $43.53_{\pm0.21}$ & $44.52_{\pm0.31}$\\
\hline
\textsc{ER} & $74.19_{\pm0.85}$ & $84.27_{\pm0.57}$ & $42.34_{\pm0.83}$ & $55.48_{\pm1.52}$ & $39.23_{\pm0.16}$ & $45.47_{\pm0.44}$\\
\rowcolor{gray!40} \textsc{DGC-ER} & $76.09_{\pm0.62}$ & $\underline{86.42}_{\pm0.58}$ & $44.46_{\pm1.07}$ & $59.55_{\pm0.97}$ & $41.38_{\pm0.52}$ & $47.40_{\pm0.45}$\\
\hline
\textsc{DER++} & $59.66_{\pm1.32}$ & $66.81_{\pm0.19}$ & $47.03_{\pm0.55}$ & $55.22_{\pm0.54}$ & $32.20_{\pm0.75}$ & $40.89_{\pm0.37}$\\
\rowcolor{gray!40} \textsc{DGC-DER++} & $62.92_{\pm0.90}$ & $67.43_{\pm0.25}$ & $49.59_{\pm1.06}$ & $57.05_{\pm0.67}$ & $33.67_{\pm0.73}$ & $41.76_{\pm0.53}$\\
\hline
\textsc{XDER} & $70.12_{\pm0.68}$ & $70.35_{\pm0.63}$ & $61.45_{\pm0.50}$ & $66.51_{\pm0.42}$ & $52.45_{\pm0.92}$ & $55.12_{\pm0.22}$\\
\rowcolor{gray!40} \textsc{DGC-XDER} & $72.34_{\pm1.08}$ & $72.41_{\pm1.05}$ & $\underline{62.70}_{\pm0.44}$ & $\underline{67.59}_{\pm0.18}$ & $\underline{53.50}_{\pm0.25}$ & $\underline{55.94}_{\pm0.19}$\\
\hline
\textsc{Dynamic ER} & $79.65_{\pm0.86}$ & $83.30_{\pm0.93}$ & $61.92_{\pm2.75}$ & $64.57_{\pm2.02}$ & $54.88_{\pm1.64}$ & $56.70_{\pm0.73}$\\
\rowcolor{gray!40} \textsc{DGC-Dynamic ER} & $\mathbf{84.23}_{\pm1.62}$ & $\mathbf{89.90}_{\pm0.93}$ & $\mathbf{63.33}_{\pm1.26}$ & $\mathbf{70.70}_{\pm1.31}$ & $\mathbf{58.10}_{\pm1.06}$ & $\mathbf{58.23}_{\pm0.84}$\\
\toprule[2pt]
\end{tabular}
\caption{The FAIA $\pm$ standard error(\%) in CIL. The methods combined with DGC are colored in gray. The best results are highlighted in bold, and the best results except Dynamic ER and DGC-Dynamic ER are underlined. Since  AOP  does not store previous data during training, it only has one numerical result per dataset in the table (without specifying the buffer size).}
\label{tab:main}
\end{table*}

\begin{table}[t]
\centering
\begin{tabular}{ccc} 
\toprule[2pt]
\textbf{Datasets} &
\multicolumn{2}{c}{\textbf{S-CIFAR100}} \\
 & $5$ Tasks & $20$ Tasks\\
\hline
\textsc{AOP} & $43.31_{\pm0.44}$ & $40.99_{\pm0.36}$\\
\textsc{AGEM} & $38.67_{\pm0.14}$ & $16.21_{\pm0.06}$\\
\textsc{SSVRG} & $46.00_{\pm4.60}$ & $33.22_{\pm4.10}$\\
\textsc{MOCA} & $66.58_{\pm0.14}$ & $22.05_{\pm2.56}$\\
\textsc{GSS} & $50.43_{\pm0.43}$ & $25.93_{\pm0.11}$\\
\textsc{GCR} & $67.20_{\pm0.38}$ & $51.57_{\pm1.63}$\\
\textsc{HAL} & $35.05_{\pm0.68}$ & $27.16_{\pm1.51}$\\
\textsc{ICARL} & $67.45_{\pm0.27}$ & $55.77_{\pm0.57}$\\
\hline
\textsc{ER} & $60.85_{\pm0.60}$ & $52.16_{\pm0.90}$\\
\rowcolor{gray!40} \textsc{DGC-ER} & $62.13_{\pm0.33}$ & $54.86_{\pm1.22}$\\
\hline
\textsc{DER++} & $56.27_{\pm0.27}$ & $53.77_{\pm1.56}$\\
\rowcolor{gray!40} \textsc{DGC-DER++} & $57.45_{\pm1.04}$ & $57.71_{\pm0.35}$\\
\hline
\textsc{XDER} & $67.14_{\pm0.38}$ & $60.45_{\pm0.46}$\\
\rowcolor{gray!40} \textsc{DGC-XDER} & $\mathbf{67.56}_{\pm0.79}$ & $\mathbf{63.53}_{\pm0.48}$\\
\toprule[2pt]
\end{tabular}

\caption{The FAIA $\pm$ standard error(\%) with different number of tasks. DGC methods are colored in gray. The best results are highlighted in bold.}
\label{tab:different Tasks}
\vskip -0.2in
\end{table}

\section{Experiments}
\label{Sec: Experiment}

We conduct the experiments to compare with
various baseline methods across different datasets. We consider both the CIL and TFCL models. 




  \textbf{Datasets} We carry out the experiments on three widely employed datasets S(Split)-CIFAR10, S-CIFAR100 \cite{krizhevsky2009learning}, and S-TinyImageNet \cite{le2015tiny}. S-CIFAR10 is the split dataset by 
partitioning CIFAR10 into $5$ tasks, each containing two categories; similarly, S-CIFAR100 and S-TinyImageNet are the datasets by respectively partitioning 
 CIFAR100 and TinyImageNet  into $10$ tasks, each containing $10$ (S-CIFAR100) and   $20$ (S-TinyImageNet) categories. 

  \textbf{Baseline methods} We consider the following baselines. \textbf{(1) Replay-based methods:} ER~\cite{chaudhry2019tiny}, DER++
  \cite{buzzega2020dark}, XDER \cite{boschini2022class}, MOCA \cite{DBLP:journals/tmlr/YuHHLWL23}, GSS \cite{aljundi2019gradient}, GCR \cite{tiwari2022gcr}, HAL \cite{chaudhry2021using}, and ICARL \cite{rebuffi2017icarl}. \textbf{(2) Optimization-based methods:} AGEM \cite{DBLP:conf/iclr/ChaudhryRRE19}, AOP \cite{guo2022adaptive}, and SSVRG \cite{frostig2015competing}.
  \textbf{(3) Dynamic architecture method:} Dynamic ER \cite{yan2021dynamically}.  We integrate DGC with ER, DER++, XDER, and Dynamic ER, and assess their performances in CIL. For  TFCL, we consider its combination with ER and DER++.
For convenience, we use ``DGC-Y'' to denote the combination of  
 DGC  with a CL method ``Y''. For example, DGC-ER denotes the method combining ER and DGC methods.


  \textbf{Evaluation metrics}
We employ the {\em Average Accuracy (AA)} \cite{chaudhry2018riemannian, mirzadeh2020linear} and the {\em final Average Incremental Accuracy (FAIA)} \cite{hou2019learning, douillard2020podnet} to assess the performance. These two metrics are both widely used for continual learning. Let $a_{k,j} \in [0,1] (k\ge j)$ denote the classification accuracy evaluated on the testing set of the task $\mathcal{T}_j$ after learning  $\mathcal{T}_k$. The value AA at time spot $i$ is defined as
    $\text{AA}_i \triangleq \frac{1}{i}\sum_{j=1}^{i} a_{i,j}$. 
In particular, we  name the value $\text{AA}_T$ as the {\em final Average Accuracy (FAA)}.
The final AIA is defined as 
    $\text{FAIA} \triangleq \frac{1}{T} \sum_{i=1}^{T} \text{AA}_i$.
We also use {\em Final Forgetting (FF)} from~\cite{chaudhry2018riemannian} to measure the forgetting of the model throughout the learning process
(the formal definition of FF and its numerical results are placed in the supplement).
Each instance of our experiments is repeated by 10 times.


\begin{figure*}[t]
\centering
\includegraphics[width=0.95\linewidth, height=6.5cm]{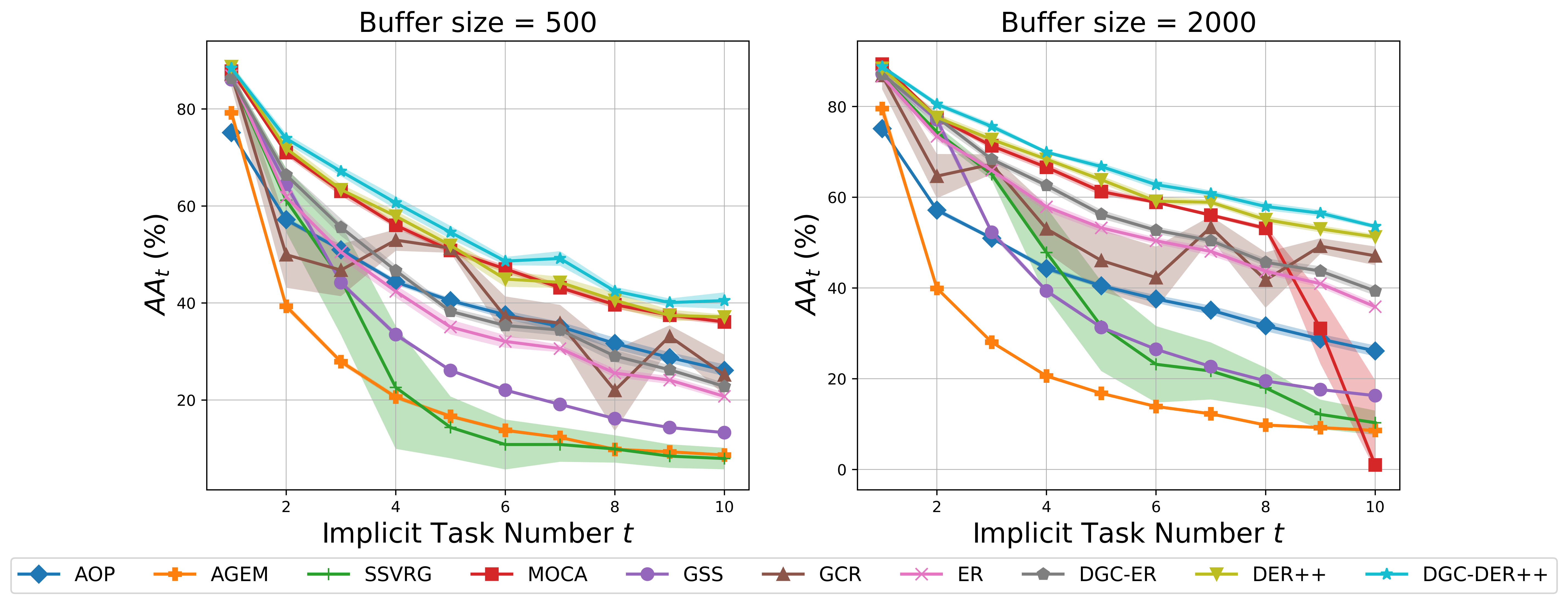} 

\caption{The averaged $AA_t$ over implicit task number $t$ in TFCL on S-CIFAR100.}
\label{fig:taskfree}  

\end{figure*}

\subsection{Results in CIL}
\label{subsec:CIL results}

{  \textbf{Hyper-parameters selection} In our implementation, we fixed the values of epoch and batch size, which implies that the total number of optimization steps (i.e., the value $s\times m$) is also fixed. In our experiments, we set the value $m=200$, so the value of $s$ in Algorithm~\ref{alg} is also determined.}

 Except for the Dynamic ER (which is dynamically expanded), all other testing methods have constant storage limits. The results shown in Table~\ref{tab:main} reveal that our DGC method can bring improvements to the combined methods on the testing benchmarks.
 For the sake of clarity, we also underline the best-performing method except Dynamic ER and DGC-Dynamic ER in Table~\ref{tab:main}.
Among the methods with constant storage limits, DGC-ER achieves the best results on S-CIFAR10 when the buffer size is $2000$, and DGC-XDER achieves the best results on S-CIFAR100 and S-TinyImageNet.
It is worth noting that our DGC method can also be conveniently integrated with existing advanced dynamic expansion representation techniques, such as Dynamic ER,  
which demonstrates the improvements  to certain extent, e.g.,  it achieves an improvement more than $6\%$ on S-CIFAR10/100 with buffer size $2000$. {  We also record the training time of these baseline methods in our supplement.}


  In the subsequent experiment, we investigate the performance of DGC compared to other baseline methods with varying the number of tasks. Throughout the experiment, we maintain a constant buffer size of $2000$. As outlined in Table~\ref{tab:different Tasks}, our results demonstrate that DGC can bring certain improvements to ER, DER++, and XDER with setting the number of tasks to be $5$ and $20$.
Moreover, similar to the previous research in~\cite{boschini2022class},   the results shown in Table \ref{tab:different Tasks} also imply that increasing the number of tasks could exacerbate the catastrophic forgetting issue. This phenomenon occurs because the model faces a reduced volume of data on each specific task, and thereby necessitates the capability of retaining the information of historical data to guide the model updates. 
As can be seen, the improvement obtained by DGC for the case of $20$ tasks usually is more significant compared with the case of $5$ tasks. For example, DGC-DER++ achieves a $3.94\%$ improvement with $20$ tasks versus a $1.18\%$ improvement with $5$ tasks, while DGC-XDER exhibits a $3.08\%$ improvement with $20$ tasks and only $0.42\%$ with $5$ tasks. These results highlight the advantage of DGC for mitigating catastrophic forgetting by effectively utilizing historical data through gradient-based calibration.

\subsection{Results in TFCL}
\label{subsec:task-free results}

We then conduct the experiments in TFCL. The curves shown in Figure \ref{fig:taskfree} depict the average $AA_t$ evolutions with varying the ``implicit'' task number $t$ on S-CIFAR100. We call it ``implicit'' since the value $t$ is not given during the training, which means that the design of the algorithm cannot rely on task boundaries or task identities. Therefore, we only compare those baseline methods that are applicable to TFCL model. The results suggest that DGC can bring improvements to ER and DER++ on $AA_t$ for almost all the $t$s; in particular, DGC-DER++ achieves the best performance and also with small variances. 
Through approximating the full gradient, the GCR and SSVRG methods can relieve the catastrophic forgetting issue to a certain extent, but they still suffer from the issue of storage limit, which affects their effectiveness for estimating the full gradient. Consequently, these methods exhibit performance downgrade and larger variance in Figure \ref{fig:taskfree}. Comparing with them, the approaches integrated with DGC illustrate more consistent and stable improvements.
More detailed results for TFCL are available in our supplement.

\subsection{Smoothness of Training with DGC}
\label{subsec:loss trajectories}

  Since our proposed DGC approach stems from the variance reduction method of SGD, a natural question is whether it also improves the smoothness during the training process.
  We  compare the loss trajectories on the entire training dataset before and after implementing DGC.  From Figure~\ref{fig:loss} we can see that the classical ER method has erratic fluctuations in loss during the training process; this could cause some practical problems in a real-world CL scenario, e.g., 
  we may need to pay more effort to carefully adjust the learning rate and determine the stopping condition.
  In contrast, the  DGC method has a smoother reduction in loss, and ultimately yields lower loss values.
  \begin{figure}[t]
\centering
\includegraphics[width=0.8\linewidth]{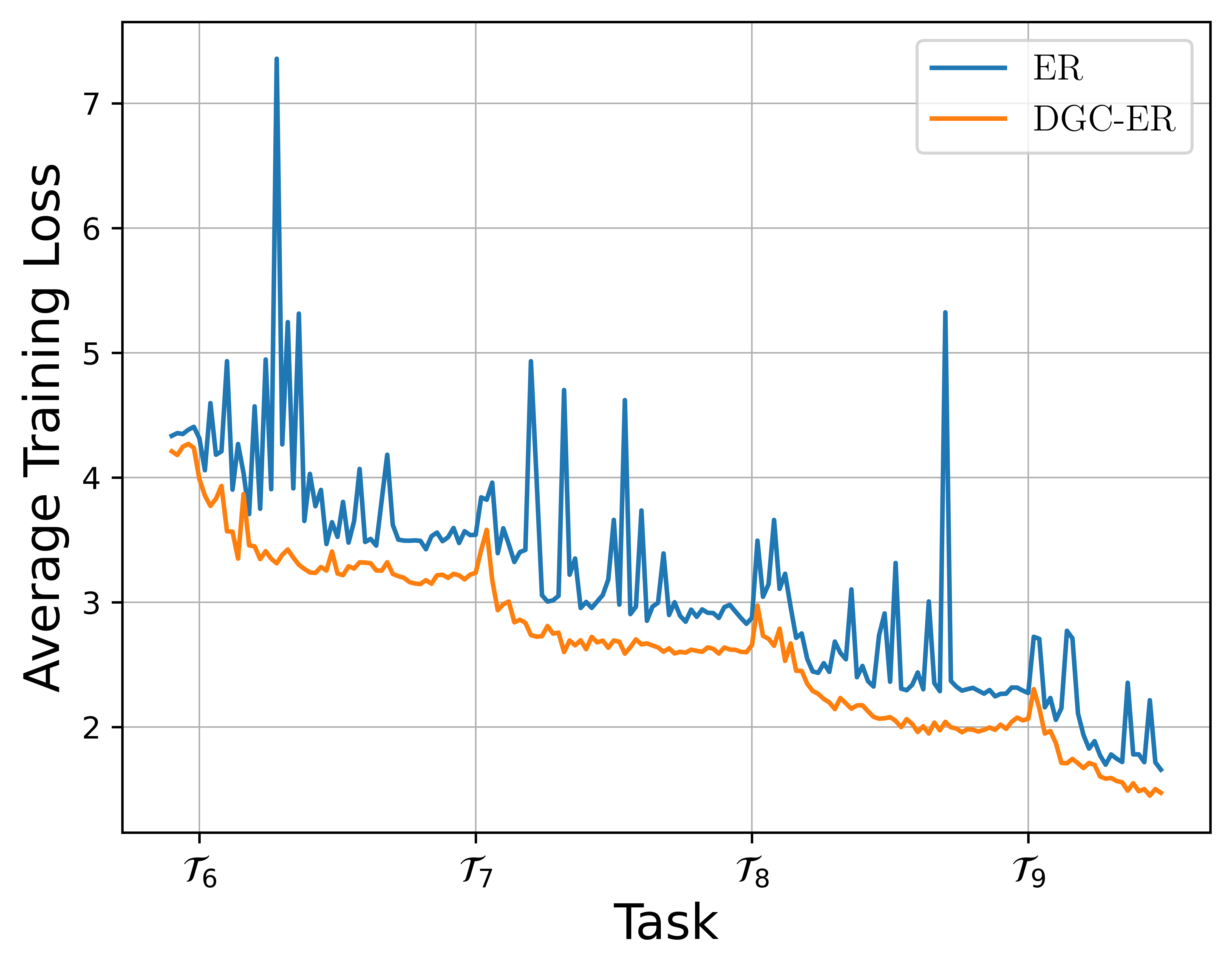} 

\caption{The partial trajectories of loss on S-CIFAR100.} 
\label{fig:loss}  
\vskip -0.2in
\end{figure}

\section{Conclusion}

In this paper, we revisit the experience replay method and aim to utilize historical information to derive a more accurate gradient for alleviating catastrophic forgetting. Inspired by the variance reduction methods for SGD, we introduce a new approach ``DGC'' to dynamically manage a gradient calibration term in CL training. Our approach can be conveniently integrated with several existing continual learning methods, contributing to a substantial improvement in both CIL and TFCL. Moreover, the improved stability of training loss reduction can also ease our practical implementation. 

 \section*{Acknowledgements}
We want to thank the anonymous reviewers and Xianglu Wang for their helpful comments. The research of this work was supported in part by the National Key Research and Development Program of China 2021YFA1000900, the National Natural Science Foundation of China 62272432, and the Natural Science Foundation of Anhui Province 2208085MF163. 
 

\section*{Impact Statement}
This paper presents work whose goal is to advance the field of Machine Learning. There are many potential societal consequences of our work, none which we feel must be specifically highlighted here.

\bibliography{example_paper}
\bibliographystyle{icml2024}

\newpage
\appendix
\onecolumn

\section{Omitted Proofs}
\label{sec:appendix_proof}

\begin{lemma}[Lemma 3.2]
\label{Appendix: gamma}
\begin{align}
 \mathbb{E} \left[\Gamma'_{\text{DGC}}(t) \right] = \mathcal{G} (\mathcal{T}_{[1:t)}, \tilde{\theta}_t ).
\end{align}
\end{lemma}

\begin{proof}
We prove it by induction. When $t = 1$,  both the two-side terms are $0$, so the equation holds. Assume the equation holds when $t = t_1$, and we consider the change of  $\Gamma'_{\text{DGC}}(t)$ at the time spot $t_1$.

Firstly, $\Gamma'_{\text{DGC}}(t_1)$ is updated at the end of every stage.  After the update, $\Gamma'_{\text{DGC}}(t_1)=\Gamma_{\text{DGC}}(t_1,m+1)$. Meanwhile, we have
\begin{align}
    &  \mathbb{E}[\Gamma_{\text{DGC}}(t_1,m+1)] \nonumber\\
& = \mathbb{E}[\nabla_\theta \ell(\bar{x}_{t_1}, \bar{y}_{t_1}, \theta_{t_1}^{m+1}) - (\nabla_\theta \ell(\bar{x}_{t_1}, \bar{y}_{t_1}, \tilde{\theta}_{t_1})  - \Gamma'_{\text{DGC}}(t_1))] \nonumber\\
&= \mathcal{G}(\mathcal{T}_{[1:t_1)},\theta_{t_1}^{m+1}) - \mathcal{G}(\mathcal{T}_{[1:t_1)},\tilde{\theta}_{t_1}) + \mathcal{G}(\mathcal{T}_{[1:t_1)},\tilde{\theta}_{t_1}) \nonumber\\
&= \mathcal{G}(\mathcal{T}_{[1:t_1)},\theta_{t_1}^{m+1}) \nonumber\\
&= \mathcal{G}(\mathcal{T}_{[1:t_1)},\tilde{\theta}_{t_1}) \nonumber ,
\end{align}
where the first equation comes from Eq (\ref{eq: DGC inner update}) and the second equation is based on the inductive hypothesis.

  Secondly, in the update from  time spot $t$ to $t+1$,  we have
\begin{align}
 \mathbb{E}[\Gamma'_{\text{DGC}}&(t_1+1)] \nonumber \\
&= \mathbb{E}[\frac{1}{t_1}((t_1-1) \cdot \Gamma_{\text{DGC}}(t_1,m+1) +  \mathcal{G}(\mathcal{T}_{t_1}, \tilde{\theta}_{t_1}) )] \nonumber \\
&= \frac{1}{t_1}((t_1-1) \cdot \mathcal{G}(\mathcal{T}_{[1:t_1)},\tilde{\theta}_{t_1}) +  \mathcal{G}(\mathcal{T}_{t_1}, \tilde{\theta}_{t_1}) ) \nonumber \\
&= \mathcal{G}(\mathcal{T}_{[1:t_1)},\tilde{\theta}_{t_1+1}) \nonumber
\end{align}
 based on Eq (\ref{eq: DGC outer update}).
This implies that our conclusion also holds when $t = t_1+1$.
\end{proof}

\begin{theorem}[Theorem~3.3]
Assume that $f(x;\theta)$ is  $L$-smooth and $\gamma$-strongly convex; the parameters 
$m \geq \frac{10L^2}{\gamma^2}$ and $\eta = \frac{\gamma}{10L}$. Then we have a linear convergence in expectation for the DGC procedure at time $t$:
$$\mathbb{E}\Big[\|\tilde{\theta}_{t,s+1}-\theta_{*}\|_{2}^{2}\Big] \leq \frac{1}{2^s} \mathbb{E}\Big[\|\tilde{\theta}_{t,1}-\theta_{*}\|_{2}^{2}\Big]$$
where $\tilde{\theta}_{t,s}$ represents the initialization parameter at the beginning of the $s$-th stage at time spot $t$.
\end{theorem}
Our proof  is inspired by the idea of \cite{johnson2013accelerating}. To simplify the notations, we define 
$$v_t^k = \nabla \psi_{x_{t}}\left(\theta_t^{k}\right)-\nabla \psi_{x_{t}}(\tilde{\theta}_t)+\tilde{\mu}$$ through combining Eq (\ref{eq: final update}) and (\ref{eq: DGC inner update}) (we replace ``$\Gamma$'' in Eq (\ref{eq: final update}) by (\ref{eq: DGC inner update})), where
\begin{align}
\nabla \psi_{x_{t}}\left(\theta_t^{k}\right) & \triangleq \frac{1}{t} \nabla_\theta \ell(x_t,y_t,\theta_t^k) + \frac{t-1}{t} \nabla_\theta \ell(\bar{x},\bar{y},\theta_t^k) \nonumber, \\
\nabla \psi_{x_{t}}(\tilde{\theta}_t) &\triangleq \frac{1}{t} \nabla_\theta \ell( {x}_t, {y}_t, \tilde{\theta}_t) + \frac{t-1}{t} \nabla_\theta \ell(\bar{x},\bar{y},\tilde{\theta}_t) \nonumber, \\
\tilde{\mu} &\triangleq \frac{1}{t} \mathcal{G}(\mathcal{T}_t, \tilde{\theta}_t) + \frac{t-1}{t} \Gamma^\prime_{\text{DGC}}(t) \nonumber.
\end{align}

Before proving Theorem 3.3, we provide the following key lemmas first.

\begin{lemma} \label{lemma: ex of v_t^k}
\begin{equation}
\begin{aligned} \nonumber
\mathbb{E}[\left\|v_{t}^{k}\right\|_{2}^{2}] \leq 2 L\left[\mathbb{E}[||\theta_{t}^{k}-\theta_{*}||_{2}^{2}]+\mathbb{E}[||\tilde{\theta}_t-\theta_{*}||_{2}^{2}]\right].
\end{aligned}
\end{equation}

\begin{proof}
\begin{equation} \label{Appendix: ineq for v}
\begin{aligned}
& \mathbb{E}[\left\|v_{t}^{k}\right\|_{2}^{2}] \\
\leq & 2 \mathbb{E}\left[\| \nabla \psi_{x_{t}}\left(\theta_{t}^{k}\right)-\nabla \psi_{x_{t}}\left(\theta_{*}\right) \|_{2}^{2}\right]\\
& +2 \mathbb{E}\left[\| \nabla \psi_{x_{t}}(\tilde{\theta_t})-\nabla \psi_{x_{t}}\left(\theta_{*}\right)  -\nabla \ell_{\text{CL}}^{t}(\tilde{\theta}_t)\|_{2}^{2}\right] \\
= & 2 \mathbb{E}\left[ \|\nabla \psi_{x_{t}}\left(\theta_{t}^{k}\right)-\nabla \psi_{x_{t}}\left(\theta_{*}\right)\|_{2}^{2}\right]\\
&+2 \mathbb{E} [\|\nabla \psi_{x_{t}}(\tilde{\theta}_t)-\nabla \psi_{x_{t}}\left(\theta_{*}\right) \\
&-\mathbb{E}[\nabla \psi_{x_{t}}(\tilde{\theta}_t)-\nabla \psi_{x_{t}}\left(\theta_{*}\right)] \|_{2}^{2}] \\
\leq & 2 \mathbb{E}[\left\|\nabla \psi_{x_{t}}\left(\theta_{t}^{k}\right)-\nabla \psi_{x_{t}}\left(\theta_{*}\right)\right\|_{2}^{2}]\\
&+2 \mathbb{E}[\left\|\nabla \psi_{x_{t}}(\tilde{\theta}_t)-\nabla \psi_{x_{t}}\left(\theta_{*}\right)\right\|_{2}^{2}] \\
\leq & 2 L\left[\mathbb{E}[||\theta_{t}^{k}-\theta_{*}||_{2}^{2}]+\mathbb{E}[||\tilde{\theta}_t-\theta_{*}||_{2}^{2}]\right], 
\end{aligned}
\end{equation}
where the first inequality comes from the fact  $\|a+b\|_{2}^{2} \leq 2\|a\|_{2}^{2}+2\|b\|_{2}^{2}$  for any vectors $a,b$  and  Lemma \ref{Appendix: gamma},  the second inequality is based on the fact  $\mathbb{E}[\|\xi-\mathbb{E} \xi\|_{2}^{2}]=\mathbb{E}[\|\xi\|_{2}^{2}]-\|\mathbb{E} [\xi]\|_{2}^{2} \leq \mathbb{E}[\|\xi\|_{2}^{2}]$  for any random vector  $\xi$, and the third inequality is based on $L$-smooth.
\end{proof}
\end{lemma}

\begin{lemma} \label{lemma: last_new}
\begin{equation}
\begin{aligned}
& \mathbb{E}[\left\|\theta_{t}^{k}-\theta_{*}\right\|_{2}^{2}] \leq (1-2\eta\gamma+2L\eta^2)\mathbb{E}[||\theta_{t}^{k-1}-\theta_{*}||_{2}^{2}] \\
&+ 2L\eta^2\mathbb{E}[||\tilde{\theta}_{t}-\theta_{*}||_{2}^{2}]. \label{appendix: theta_new}
\end{aligned}
\end{equation}

\begin{proof}
Note that $\mathbb{E} [v_{t}^k]= \mathbb{E} [\nabla \psi_{x_{t}}(\theta_t^{k})] =\nabla \ell_{\text{CL}}^{t}\left(\theta_{t}^{k-1}\right)$, and so we have
\begin{equation} 
\begin{aligned}
& \mathbb{E}[\left\|\theta_{t}^{k}-\theta_{*}\right\|_{2}^{2}] \\
= & \left\|\theta_{t}^{k-1}-\theta_{*}\right\|_{2}^{2}-2 \eta\left(\theta_{t}^{k-1}-\theta_{*}\right)^{\top} \mathbb{E} [v_{t}^{k-1}]+\eta^{2} \mathbb{E}[\left\|v_{t}^{k-1}\right\|_{2}^{2}] \\
\leq & \left\|\theta_{t}^{k-1}-\theta_{*}\right\|_{2}^{2}-2 \eta\left(\theta_{t}^{k-1}-\theta_{*}\right)^{\top} \nabla \ell_{\text{CL}}^{t}\left(\theta_{t}^{k-1}\right)\\
& +2 L \eta^{2}\left[\mathbb{E}[||\theta_{t}^{k-1}-\theta_{*}||_{2}^{2}]+\mathbb{E}[||\tilde{\theta}_t-\theta_{*}||_{2}^{2}]\right] \\
\leq & \left\|\theta_{t}^{k-1}-\theta_{*}\right\|_{2}^{2}-2 \eta\gamma ||\theta_{t}^{k-1}-\theta_{*}||_{2}^{2}\\
& +2 L \eta^{2}\left[\mathbb{E}[||\theta_{t}^{k-1}-\theta_{*}||_{2}^{2}]+\mathbb{E}[||\tilde{\theta}_t-\theta_{*}||_{2}^{2}]\right], \nonumber
\end{aligned}
\end{equation}
where the first inequality comes from  Lemma~\ref{lemma: ex of v_t^k}, and the second inequality is based on the $\gamma$-strong convexity of  $\ell_{\text{CL}}^{t}(\theta)$. We take the expectation on $\theta_t^{k-1}$ from the above inequality and then have
\begin{align}
& \mathbb{E}[\left\|\theta_{t}^{k}-\theta_{*}\right\|_{2}^{2}] \nonumber\\
&\leq  \mathbb{E} [\left\|\theta_{t}^{k-1}-\theta_{*}\right\|_{2}^{2}]-2 \eta\gamma \mathbb{E}[||\theta_{t}^{k-1}-\theta_{*}||_{2}^{2}] \nonumber\\
& +2 L \eta^{2}\left[\mathbb{E}[||\theta_{t}^{k-1}-\theta_{*}||_{2}^{2}]+\mathbb{E}[||\tilde{\theta}_t-\theta_{*}||_{2}^{2}]\right] \nonumber \\
& = (1-2\eta\gamma+2L\eta^2)\mathbb{E}[||\theta_{t}^{k-1}-\theta_{*}||_{2}^{2}] + 2L\eta^2\mathbb{E}[||\tilde{\theta}_{t}-\theta_{*}||_{2}^{2}].\nonumber
\end{align}
\end{proof}
\end{lemma}

\begin{proof}[\textbf{Proof of Theorem~3.3}]
We now proceed to prove the theorem. To simplify notation, we denote $\mathbb{E}\left\|\theta_{t}^{k}-\theta_{*}\right\|_{2}^{2}$ as $p_k$.
Based on lemma~\ref{lemma: last_new}, we have
\begin{align}
p_{k} \leq\left(1-2 \eta \gamma+2 \eta^{2} L\right) p_{k-1}+4 \eta^{2} L  p_{k-1}. \nonumber
\end{align}
By setting  $\eta = \frac{\gamma}{10L}$ and $m \geq \frac{10L^2}{\gamma^2}$, we have
\begin{align}
p_{m+1} &\leq\left(1-\frac{\gamma^{2}}{10 L^{2} }\right) p_{m} \leq\left(1-\frac{\gamma^{2}}{10 L^{2} }\right)^{m} p_{1} \nonumber\\
 &\leq \exp \left(-\frac{\gamma^{2}}{10 L^{2} } m\right) p_{1} \leq \frac{1}{2} p_{1}, \nonumber
\end{align}
i.e
\begin{align} \label{a}
\mathbb{E}\left[\|\theta_{t}^{m+1}-\theta_{*}\|_{2}^{2}\right] \leq \frac{1}{2} \mathbb{E}\left[\|\theta_{t}^{1}-\theta_{*}\|_{2}^{2}\right].
\end{align}
Subsequently, we consider the update during different stage $\mathrm{s}$, so that the initial parameter $\tilde{\theta}_{t,s}$ is $ \theta_t^{1}$  and updated parameter $ \tilde{\theta}_{t,s+1} = \theta_{t}^{m+1}$  is selected after all of the updates have been completed. Then Eq (\ref{a}) becomes
\begin{align} 
\mathbb{E}\left[\|\tilde{\theta}_{t,s+1}-\theta_{*}\|_{2}^{2}\right] \leq \frac{1}{2} \mathbb{E}\left[\|\tilde{\theta}_{t,s}-\theta_{*}\|_{2}^{2}\right].
\end{align}
Through recursively applying  the above equation from stage $1$ to stage $s+1$, we have 
$$\mathbb{E}\left[\|\tilde{\theta}_{t,s+1}-\theta_{*}\|_{2}^{2}\right] \leq \frac{1}{2^s} \mathbb{E}\left[\|\tilde{\theta}_{t,1}-\theta_{*}\|_{2}^{2}\right].$$
\end{proof}

\section{Detailed DGC procedure in TFCL}
We consider the data streams $\{\mathcal{T}_1, \mathcal{T}_2, \cdots\, \mathcal{T}_T\}$ with each $\mathcal{T}_i$ being the batch training data in TFCL. Algorithm~\ref{DGC in TFCL} shows how to implement the DGC procedure in this model. 


\begin{algorithm}[ht]
    \caption{DGC procedure in TFCL}
    \label{DGC in TFCL}
    \begin{algorithmic}[1]
        \STATE {\bfseries Input:} Data stream $\{\mathcal{T}_1, \mathcal{T}_2, \cdots\, \mathcal{T}_T\}$, update steps $m$, batch size $b$, and learning rate $\eta$.
        \STATE {\bfseries Output:} Trained model parameter $\theta_T$

        \STATE Initialize model parameters $\theta_0$, $\Gamma^\prime_{\text{DGC}}(1) = \Vec{0}$, $\tilde{t} = 1$
        \STATE Initialize buffer $\mathcal{M}_1 = \emptyset$, $\tilde{\theta} = \theta_0$
        \FOR{$t = {1,2,\ldots, T}$}
            \STATE $\theta_t \gets \theta_{t-1}$
                \STATE Take a uniform sample $X'_t$ of size $b$ from $M_t$
                \STATE Calculate $\Gamma_{\text{DGC}}(t,1)$ with $\Gamma^\prime_{\text{DGC}}(\tilde{t})$ according to   (\ref{eq: DGC inner update})
                
                \renewcommand{\algorithmiccomment}[1]{{\color{blue}\textit{/* #1 */}}}
                \STATE Calculate $v_t^1$ with $\Gamma_{\text{DGC}}(t,1)$  according to (\ref{eq: final update})
                
                \COMMENT{Calculate the calibrated gradient}
                \STATE $\theta_t \gets \theta_t - \eta \cdot v_t^1$
                \renewcommand{\algorithmiccomment}[1]{{\color{blue}\textit{/* #1 */}}}

            \IF{$(t-1) \mod m = 0$} 
                \STATE Update $\Gamma^\prime_{\text{DGC}}(\tilde{t})$ according to  (\ref{eq: DGC inner update}) and (\ref{eq: stage update})
                \renewcommand{\algorithmiccomment}[1]{{\color{blue}\textit{/* #1 */}}}
                    
                \COMMENT{Update $\Gamma^\prime_{\text{DGC}}(\tilde{t})$ from $\Gamma_{\text{DGC}}(t,1)$}

                \STATE $\tilde{t} \gets t$, $\tilde{\theta} \gets \theta_t$
            \ENDIF 

            \STATE $\mathcal{M}_{t+1} \gets$ MemoryUpdate($\mathcal{T}_t$, $\mathcal{M}_t$)
            \renewcommand{\algorithmiccomment}[1]{{\color{blue}\textit{/* #1 */}}}
            
            \COMMENT{Reservoir sampling}

            \STATE Calculate and store $\Gamma^\prime_{\text{DGC}}(\tilde{t})$ according to (\ref{eq: DGC outer update})
            
            \COMMENT{Update $\Gamma^\prime_{\text{DGC}}(\tilde{t})$ for new historical data }
        \ENDFOR 
    \end{algorithmic}
\end{algorithm}

\begin{table*}[htbp]
\centering
\begin{tabular}{ccccc} 
\toprule[2pt]
\textbf{Datasets} &
\multicolumn{4}{c}{\textbf{S-CIFAR100}} \\
$\alpha$ & $1e^{-2}$ & $1e^{-3}$ & $1e^{-4}$ & $0$\\
\hline
\textsc{DGC-ER} & $57.4_{\pm2.14}$ & $59.55_{\pm0.97}$ & $59.14_{\pm0.91}$ & $55.48_{\pm1.52}$\\
\toprule[2pt]
\end{tabular}
\caption{The FAIA $\pm$ standard error(\%) with different selection of $\alpha$.}
\label{tab:alpha}
\end{table*}

\begin{table*}[htbp]
\centering
\begin{tabular}{cccccc} 
\toprule[2pt]
\textbf{Datasets} &
\multicolumn{4}{c}{\textbf{S-CIFAR100}} \\
$m$ & $1$ & $100$ & $200$ & $300$ & $400$\\
\hline
\textsc{DGC-ER} & $49.11_{\pm5.30}$ & $59.44_{\pm0.95}$ & $59.55_{\pm0.97}$ & $59.27_{\pm0.65}$ & $58.94_{\pm0.98}$ \\
\toprule[2pt]
\end{tabular}
\caption{The FAIA $\pm$ standard error(\%) with different selection of $m$.}
\label{tab:m}
\end{table*}

\section{Details for Experimental Setup}
\label{sec:detailed setup}
  \textbf{Architecture and hyperparameter} Most of the comparison methods in our experiments use the implementation of Mammoth~\footnote{https://github.com/aimagelab/mammoth}. For the Dynamic ER \cite{yan2021dynamically} method we use the implementation of PyCIL~\footnote{https://github.com/G-U-N/PyCIL}, and for the GCR method we use the implementation of Google-research~\footnote{https://github.com/google-research/google-research/tree/master/gradient\_coresets\_replay}. To be consistent with the selection of these open source frameworks, all the methods use ResNet18~\cite{he2016deep} as the base network. All the networks are trained from scratch. 
For the hyperparameter selection of different methods, we directly use the original hyperparameters used in these open-source frameworks, which are obtained by using grid search on $10\%$ of the training set.

  \textbf{Training details} 
  To maintain consistency, we utilize the default parameter settings and network architectures provided by the framework. As in previous studies~\cite{tiwari2022gcr, yan2021dynamically, buzzega2020dark}, random crops and horizontal flips are used as data augmentation in all experiments. All methods that do not incorporate DGC are optimized using standard SGD. In our DGC calibrated method, we empirically set $m = 200$. 
  To fairly compare the methods with constant storage limits, we uniformly train for $50$ epochs in each task on S-CIFAR10 and S-CIFAR100, and $100$ epochs in each task on S-TinyImageNet. The batch size is all set to  be $32$. According to the experimental description in the Dynamic ER method, 
  we train the first task for $200$ epochs on all datasets and train all subsequent tasks for $170$ epochs;    the batch size is set to be $128$.

\section{Other Experimental Results}
\label{sec:detailed exp}

  \textbf{Selection of $\boldsymbol{\alpha}$} We explore the impact of different $\alpha$ values to the performance of our DGC method. In the experiment, we fix the buffer size to $2000$. Table~\ref{tab:alpha} shows that DGC can improve ER under different $\alpha$ values. 
  In all the experiments of \cref{Sec: Experiment}, we fix the value of $\alpha$ to be $1e^{-3}$.

  \textbf{Selection of $\boldsymbol{m}$} As stated in Theorem~\ref{th:main}, when $m$ is sufficiently large, our DGC method has a linear convergence in expectation. The results in Table~\ref{tab:m}  show that when $m$ is greater than $100$, DGC-ER can significantly improve the performance of ER. Table~\ref{tab:m} also suggests that our method is relatively robust to the selection of $m$. 

\begin{table*}[htbp]
\centering
\begin{tabular}{ccccccc}
\toprule[2pt]
\textbf{Datasets} &
\multicolumn{2}{c}{\textbf{S-CIFAR10}} &
\multicolumn{2}{c}{\textbf{S-CIFAR100}} & 
\multicolumn{2}{c}{\textbf{S-TinyImageNet}}\\
Size of Buffer & $500$ & $2000$ & $500$ & $2000$ & $2000$ & $5000$ \\
\hline
\textsc{AOP} & \multicolumn{2}{c}{$49.54_{\pm1.44}$
} & \multicolumn{2}{c}{$14.24_{\pm2.18}$} & \multicolumn{2}{c}{$3.92_{\pm
0.19}$}\\
\textsc{AGEM} & $96.50_{\pm1.29}$ & $96.34_{\pm0.87}$ & $89.58_{\pm0.14}$ & $89.56_{\pm0.22}$ & $77.30_{\pm0.22}$ & $77.53_{\pm0.39}$ \\
\textsc{SSVRG} & $52.10_{\pm7.83}$ & $68.75_{\pm11.86}$ & $56.65_{\pm5.21}$ & $70.15_{\pm8.61}$ & $21.99_{\pm2.02}$ & $31.29_{\pm3.92}$ \\
\textsc{MOCA} & $12.97_{\pm1.00}$ & $7.22_{\pm0.71}$ & $28.65_{\pm0.87}$ & $58.08_{\pm17.70}$ & $41.61_{\pm14.34}$ & $32.04_{\pm11.13}$\\
\textsc{GSS} & $64.81_{\pm4.99}$ & $47.27_{\pm5.06}$ & $83.43_{\pm0.50}$ & $79.52_{\pm0.43}$ & $72.75_{\pm0.40}$ & $69.97_{\pm0.48}$ \\
\textsc{GCR} & $24.18_{\pm1.76}$ & $13.34_{\pm1.09}$ & $58.03_{\pm5.80}$ & $34.41_{\pm2.56}$ & $48.47_{\pm1.60}$ & $39.34_{\pm1.77}$ \\
\textsc{HAL} & $62.86_{\pm2.80}$ & $36.65_{\pm2.88}$ & $55.22_{\pm1.62}$ & $47.69_{\pm2.73}$ & $43.72_{\pm2.34}$ & $38.86_{\pm1.25}$ \\
\textsc{ICARL} & $31.91_{\pm2.25}$ & $26.56_{\pm1.19}$ & $30.20_{\pm0.40}$ & $24.64_{\pm0.33}$ & $18.55_{\pm0.45}$ & $17.62_{\pm0.61}$ \\
\hline
\textsc{ER} & $44.28_{\pm1.93}$ & $22.96_{\pm0.90}$ & $74.29_{\pm0.73}$ & $54.60_{\pm0.63}$ & $65.74_{\pm0.59}$ & $54.27_{\pm0.68}$ \\
\rowcolor{gray!40} \textsc{DGC-ER} & $39.72_{\pm1.91}$ & $20.37_{\pm1.03}$ & $72.31_{\pm0.85}$ & $52.30_{\pm1.09}$ & $64.58_{\pm0.52}$ & $52.38_{\pm0.68}$ \\
\hline
\textsc{DER++} & $9.24_{\pm3.01}$ & $6.82_{\pm0.60}$ & $14.94_{\pm2.23}$ & $9.19_{\pm1.09}$ & $7.86_{\pm3.24}$ & $6.75_{\pm0.84}$ \\
\rowcolor{gray!40} \textsc{DGC-DER++} & $\mathbf{4.90}_{\pm1.86}$ & $\mathbf{3.99}_{\pm1.01}$ & $\mathbf{9.15}_{\pm2.52}$ & $\mathbf{5.79}_{\pm0.85}$ & $\mathbf{3.40}_{\pm0.66}$ & $\mathbf{4.59}_{\pm0.83}$ \\
\hline
\textsc{XDER} & $10.64_{\pm0.70}$ & $8.89_{\pm0.33}$ & $25.11_{\pm0.79}$ & $12.15_{\pm0.26}$ & $17.95_{\pm0.55}$ & $12.81_{\pm0.31}$ \\
\rowcolor{gray!40} \textsc{DGC-XDER} & $8.74_{\pm0.98}$ & $7.22_{\pm0.97}$ & $23.24_{\pm0.65}$ & $11.11_{\pm0.37}$ & $16.99_{\pm0.74}$ & $11.84_{\pm0.47}$\\
\toprule[2pt]
\end{tabular}
\caption{The FF $\pm$ standard error(\%) in CIL. The methods combined with DGC are colored in gray. The best results are highlighted in bold.}
\label{tab:FF}
\end{table*}

\begin{table*}[htbp]
\centering
\begin{tabular}{ccccccc}
\toprule[2pt]
\textbf{Datasets} &
\multicolumn{2}{c}{\textbf{S-CIFAR10}} &
\multicolumn{2}{c}{\textbf{S-CIFAR100}} & 
\multicolumn{2}{c}{\textbf{S-TinyImageNet}}\\
Size of Buffer & $500$ & $2000$ & $500$ & $2000$ & $2000$ & $5000$ \\
\hline
\textsc{AOP} & \multicolumn{2}{c}{$48.52_{\pm1.06}$
} & \multicolumn{2}{c}{$26.13_{\pm1.22}$} & \multicolumn{2}{c}{$8.61_{\pm
0.02}$}\\
\textsc{AGEM} & $19.80_{\pm0.29}$ & $20.06_{\pm0.38}$ & $9.25_{\pm0.08}$ & $9.31_{\pm0.06}$ & $7.94_{\pm0.06}$ & $7.86_{\pm0.11}$\\
\textsc{SSVRG} & $32.08_{\pm6.97}$ & $43.55_{\pm11.38}$ & $9.92_{\pm2.22}$ & $17.40_{\pm5.35}$ & $0.75_{\pm0.09}$ & $7.92_{\pm2.22}$ \\
\textsc{MOCA} & $70.34_{\pm1.23}$ & $78.68_{\pm0.86}$ & $37.13_{\pm0.55}$ & $9.35_{\pm18.67}$ & $0.50_{\pm0.00}$ & $0.50_{\pm0.00}$ \\
\textsc{GSS} & $44.56_{\pm3.95}$ & $56.46_{\pm4.48}$ & $13.22_{\pm0.05}$ & $16.12_{\pm0.13}$ & $11.74_{\pm0.14}$ & $13.45_{\pm0.20}$ \\
\textsc{GCR} & $\mathbf{72.62}_{\pm0.76}$ & $\mathbf{80.95}_{\pm0.47}$ & $28.66_{\pm4.14}$ & $48.89_{\pm1.99}$ & $29.01_{\pm0.87}$ & $36.45_{\pm0.89}$ \\
\textsc{HAL} & $39.41_{\pm1.96}$ & $58.34_{\pm2.66}$ & $8.46_{\pm1.09}$ & $11.91_{\pm2.05}$ & $5.43_{\pm0.71}$ & $8.00_{\pm0.99}$ \\
\textsc{ICARL} & $57.51_{\pm2.90}$ & $69.91_{\pm0.68}$ & $45.69_{\pm0.53}$ & $52.30_{\pm0.47}$ & $30.30_{\pm0.44}$ & $31.69_{\pm0.32}$ \\
\hline
\textsc{ER} & $61.08_{\pm1.23}$ & $76.82_{\pm0.95}$ & $21.18_{\pm0.65}$ & $36.24_{\pm1.27}$ & $18.39_{\pm0.38}$ & $26.71_{\pm0.53}$ \\
\rowcolor{gray!40} \textsc{DGC-ER} & $64.42_{\pm1.35}$ & $79.46_{\pm0.61}$ & $23.11_{\pm0.76}$ & $40.23_{\pm1.05}$ & $19.52_{\pm0.55}$ & $27.96_{\pm0.54}$ \\
\hline
\textsc{DER++} & $54.89_{\pm1.42}$ & $64.84_{\pm0.40}$ & $37.56_{\pm1.03}$ & $50.42_{\pm0.68}$ & $15.34_{\pm3.53}$ & $31.09_{\pm0.59}$ \\
\rowcolor{gray!40} \textsc{DGC-DER++} & $58.03_{\pm1.10}$ & $66.02_{\pm0.51}$ & $40.48_{\pm1.14}$ & $52.14_{\pm0.64}$ & $20.82_{\pm0.91}$ & $33.72_{\pm1.63}$ \\
\hline
\textsc{XDER} & $63.88_{\pm1.43}$ & $67.86_{\pm0.84}$ & $47.45_{\pm0.65}$ & $56.88_{\pm0.53}$ & $41.41_{\pm0.36}$ & $44.66_{\pm0.09}$ \\
\rowcolor{gray!40} \textsc{DGC-XDER} & $69.46_{\pm1.82}$ & $71.40_{\pm1.69}$ & $\mathbf{49.07}_{\pm0.40}$ & $\mathbf{57.97}_{\pm0.31}$ & $\mathbf{41.58}_{\pm0.32}$ & $\mathbf{44.88}_{\pm0.40}$ \\
\toprule[2pt]
\end{tabular}
\caption{The FAA $\pm$ standard error(\%) in CIL. The methods combined with DGC are colored in gray. The best results are highlighted in bold.}
\label{tab:FAA}
\end{table*}

  \textbf{FF and FAA results in CIL} 
  we use Final Forgetting (FF)~\cite{chaudhry2018riemannian} to measure the forgetting of the model throughout the learning process.
Following the previous work, it is defined as $\text{FF} \triangleq \frac{1}{T-1} \sum_{j=1}^{T-1} \left(\max _{k \in\{1, \ldots, T-1\}} a_{k,j}-a_{T,j}\right)$.
  Table~\ref{tab:FF} and Table~\ref{tab:FAA} show that DGC based methods can reduce forgetting while improving the performance in all the cases.  For most instances, our DGC based methods achieve the best performance.

\begin{table*}[htbp]
\centering
\begin{tabular}{cccc}
\toprule[2pt]
\textbf{Datasets} &
\textbf{S-CIFAR10} &
\textbf{S-CIFAR100} & 
\textbf{S-TinyImageNet}\\
Size of Buffer & $500$ & $2000$ & $2000$ \\
\hline
\textsc{AOP} & $66.73_{\pm0.60}$ & $42.73_{\pm0.62}$ & $21.40_{\pm
0.17}$\\
\textsc{AGEM} & $50.24_{\pm0.53}$ & $23.55_{\pm0.23}$ & $18.95_{\pm0.40}$ \\
\textsc{SSVRG} & $49.02_{\pm4.09}$ & $39.09_{\pm5.53}$ & $14.39_{\pm1.47}$\\
\textsc{MOCA} & $81.01_{\pm0.97}$ & $59.29_{\pm2.97}$ & $34.74_{\pm10.08}$ \\
\textsc{GSS} & $68.81_{\pm0.98}$ & $38.54_{\pm0.39}$ & $31.38_{\pm0.11}$ \\
\textsc{GCR} & $82.31_{\pm0.43}$ & $63.18_{\pm2.16}$ & $48.94_{\pm0.44}$\\
\hline
\textsc{ER} & $73.04_{\pm0.47}$ & $55.46_{\pm0.77}$ & $37.44_{\pm0.20}$ \\
\rowcolor{gray!40} \textsc{DGC-ER} & $74.03_{\pm0.55}$ & $58.35_{\pm0.81}$ & $38.00_{\pm0.34}$ \\
\hline
\textsc{DER++} & $82.72_{\pm0.17}$ & $65.57_{\pm0.42}$ & $47.23_{\pm1.22}$ \\
\rowcolor{gray!40} \textsc{DGC-DER++} & $\mathbf{82.93}_{\pm0.14}$ & $\mathbf{66.57}_{\pm0.20}$ & $\mathbf{50.08}_{\pm1.10}$ \\
\toprule[2pt]
\end{tabular}
\caption{The FAIA $\pm$ standard error(\%) in TFCL. The methods combined with DGC are colored in gray. The best results are highlighted in bold.}
\label{tab:taskfree}
\end{table*}

\begin{table*}[ht]
    \centering
    \begin{tabular}{cccccccccc}
        \toprule[2pt]
        Methods(s) & GSS & XDER & AOP & MOCA & SSVRG & GCR & DGC & AGEM & ER \\
        \hline
        $T_1$ & $4959.29$ & $1083.86$ & $1418.49$ & $655.15$ & $216.38$ & $689.85$ & $218.46$ & $212.76$ & $214.34$ \\
        $T_2$ & $4281.39$ & $2495.60$ & $1412.48$ & $1192.89$ & $873.39$ & $690.34$ & $666.08$ & $478.99$ & $345.05$ \\
        $T_3$ & $4158.51$ & $2826.82$ & $1415.24$ & $1173.05$ & $878.89$ & $686.16$ & $668.86$ & $488.91$ & $340.36$ \\
        \toprule[2pt]
    \end{tabular}
    \caption{Training time of the first three task in CIL on S-CIFAR100.}
    \label{tab:training time}
\end{table*}

\textbf{Detailed results in TFCL} Table~\ref{tab:taskfree} shows that our DGC method can improve the performance in TFCL. The DGC-DER++ method achieves the best performance in all the cases.

{  \textbf{Training Time} We set $b=500$ and conduct experiments on S-CIFAR100 to measure the training time of several algorithms on the first three tasks in the CIL scenario. Table~\ref{tab:training time} shows that our method and GCR have roughly the same training time, which is higher than AGEM but lower than other baselines. The results on the remaining tasks and other datasets are similar. Therefore, the increase caused by our method on training time is not significant, compared with most of the baselines.  }

\end{document}